%% file: main.tex
\documentclass[runningheads]{llncs}
\usepackage[T1]{fontenc}
\usepackage{graphicx}
\usepackage[textsize=small]{todonotes}

\usepackage{url}
\newtheorem{assumption}{Assumption}

\input{headers/commands}
\input{headers/notation}
\input{headers/gls}

\begin{document}
\title{Permissive Safety Through Trusted Inference: \\ Verifiable Belief-Space Neural Safety Filters for Assured Interactive Robotics}
\titlerunning{Verifiable Belief-Space Safety Filters}
\author{Haimin Hu\orcidID{0000-0002-4217-4776}}
\authorrunning{H. Hu}

\institute{Department of Computer Science, Johns Hopkins University, USA \\
\email{haimin@cs.jhu.edu}
}
\maketitle              
\setcounter{footnote}{0}

\begin{abstract}
Autonomous robots that interact with people must make safe and efficient decisions under human-induced uncertainty, such as their preferences, goals, competency, and willingness to cooperate.
Safety filters are a popular approach for ensuring safety in interactive robotics, since their modular design separates safety from performance, allowing robots to operate safely around people with minimal impact on task efficiency.
While traditional safety filters typically operate only in the physical space, neglecting the robot's ability to learn and adapt online, the recently proposed belief-space safety filter (\beliefsf) reasons about robot safety in closed-loop with runtime inference that actively reduces the robot's uncertainty online, thereby reducing conservativeness in filtering.
However, providing formal safety guarantees for robots deploying \beliefsf~remains a significant challenge due to errors in runtime inference and neural approximation of safety filters required to handle the high dimensionality of belief spaces.
In this paper, we propose an algorithmic approach to certify high-probability safety of \beliefsf~using conformal prediction, while explicitly accounting for the reliability of the robot's runtime inference module.
Our method leverages the structure of belief-space safety filtering by focusing verification on a region where inference is expected to be reliable.
It preserves the simplicity and sample complexity of standard conformal prediction, yet can certify a substantially less conservative safety filter.
Through a simulated human--vehicle interaction benchmark, we show that our approach verifies a significantly more permissive belief-space safety filter than a standard conformal prediction baseline.

\end{abstract}
%

\section{Introduction}
\label{sec:intro}

\begin{figure}[!hbtp]
  \centering
  \includegraphics[width=1.0\columnwidth]{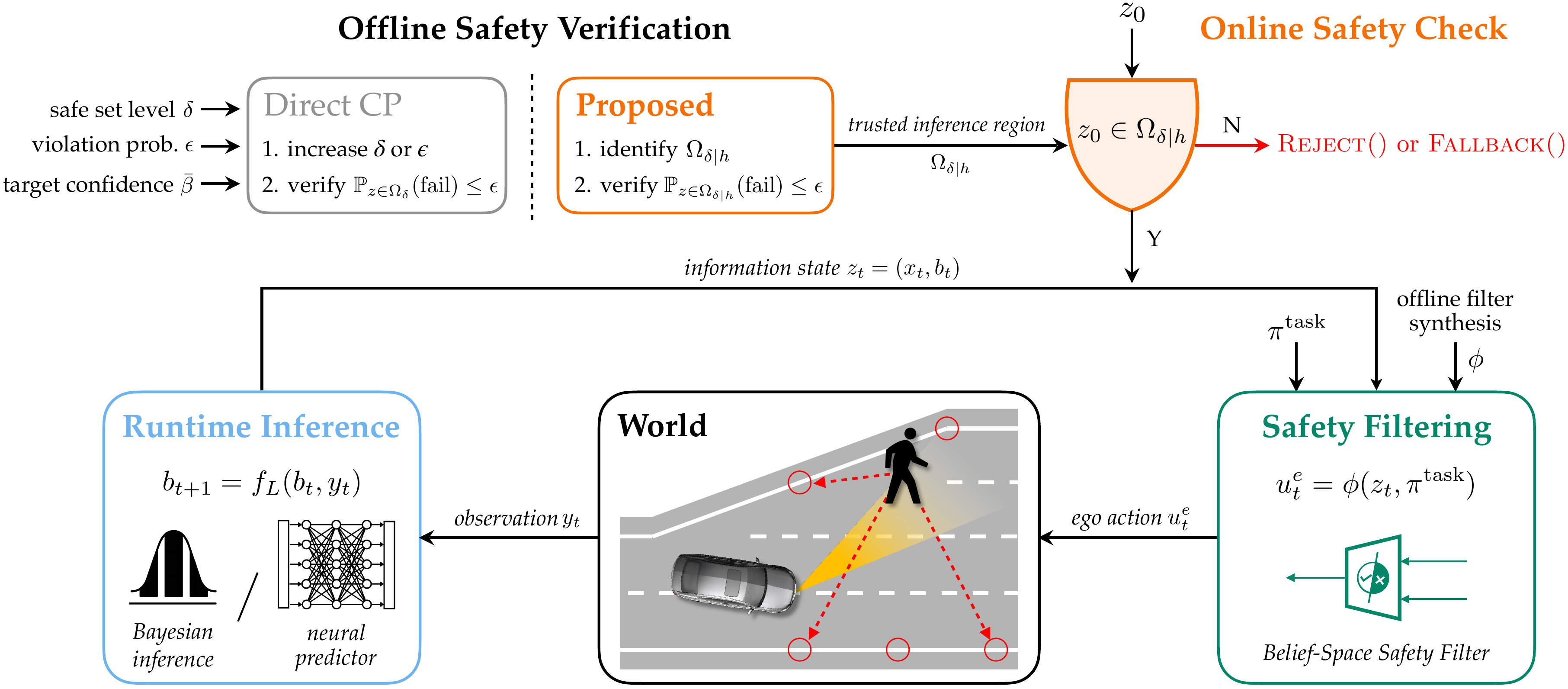}
  \caption{\label{fig:front_fig} Offline safety verification and online deployment of belief-space safety filters.
  Compared to direct conformal prediction, our proposed verification scheme jointly reasons about safety and inference, yielding a larger verified safe set and, consequently, a more permissive safety filter.
  }
\end{figure}

Autonomous robots are increasingly deployed in safety-critical and human-populated environments, from social navigation to mixed-autonomy traffic and human--robot teaming in warehouses, 
where human intentions are unobserved and their behaviors may be strongly influenced by the robot's actions.
In such settings, safety involves more than physical interactions (\eg, collision avoidance); it also depends on \emph{what the robot believes} about other agents and how it updates those beliefs online as new information is gathered during interaction.
This coupling between inference and decision-making is central to interactive robotics, yet it poses a fundamental challenge for safety: even when a control policy is nominally safe under perfect information, incorrect or delayed inference can lead to catastrophic failures.
On the other hand, if the robot's safety mechanism ignores its ability to learn from future observations and reduce uncertainty over the course of an interaction, it can become overly conservative and unnecessarily restrict performance~\cite{althoff2014online,hu2022sharp,borquez2023parameterconditioned}.

The recently proposed belief-space safety filters (\beliefsf)~\cite{hu2023deception} offer a principled way to account for the evolving uncertainty by closing the loop between the robot's safety analysis and runtime inference that updates its belief about other agents.
In particular, \beliefsf~assesses safety in an augmented state space that describes both the agents' physical behaviors and the ego robot's internal belief about other agents, thereby explicitly capturing the robot's ability to handle safety-critical events as a function of how accurately it can infer other agents' behavior through observations made during interactions (\autoref{fig:front_fig}).
While \beliefsf~can guarantee safety under assumptions on accurate inference and safety-filter synthesis, how can we certify closed-loop safety in realistic, high-dimensional interaction settings where neither inference nor filter validity is perfect?
Recent work~\cite{lin2024verification} uses \gls{CP} to verify general (non-belief) safety filters learned as neural networks and provide probabilistic safety guarantees based on trajectory rollouts.
However, directly applying \gls{CP} to verify belief-space safety filters can conflate safety failures caused by inaccurate inference with those caused by the safety filter itself (e.g., neural approximation errors), needlessly hurting the filter's \emph{permissiveness}.
For example, when the robot's runtime inference is generally accurate but occasionally fails in specific regions of the state space, standard \gls{CP} can certify a desired safety level only by shrinking the safe set \emph{uniformly}, thus restricting the robot's operation even in regions where safety could in fact be ensured.
This observation suggests that safety verification for belief-space safety filters should explicitly account for the quality of the robot's runtime inference.

\p{Contributions}
We propose a novel inference-aware verification framework for providing high-probability safety guarantees in interactive robot autonomy.
Our key insight is that, rather than directly verifying the filtered system end-to-end, \textit{focusing the filter's verification and deployment on regions where the inference is expected to be reliable can substantially reduce its conservativeness.} 
To the best of our knowledge, this paper is the first to establish closed-loop safety guarantees for interactive motion planning in belief spaces without assumptions on perfect inference or filter validity.
We show that our approach can yield a significantly larger certified safe set and more permissive filter than a standard \gls{CP} baseline through extensive simulation studies of an 18-dimensional interactive driving scenario with a strategically deceptive opponent agent.

\section{Related Work}
\label{sec:related_work}

\p{Safety Filters}
Safety filters monitor a robot's proposed task actions and, when necessary, modify them to avoid upcoming safety failures.
Representative classes of safety filters include, for example, model predictive safety filtering, which reasons about safety by rolling out or optimizing system trajectories~\cite{wabersich2018linear,bastani2021safe,strawn2023conformal,nguyenhsu2024gameplay}, control barrier functions, which smoothly modifies the proposed action via optimization to slow down the system when it approaches the boundary of the safe set~\cite{ames2017cbf,xiao2019hocbf,robey2020learning,oh2025safety}, and reachability-based safety filters, which uses a controlled-invariant set~\cite{blanchini1999set} and the associated safe control policy obtained via computing or approximating the solution to a Hamilton--Jacobi reachability problem~\cite{mitchell2005time,bansal2017hamilton,fisac2019bridging,hsu2021safety,wang2024magics,li2025certifiable}. 
Beyond motion planning, recent work~\cite{bejarano2024safety,oh2025provably} shows that safety filters can help \gls{RL} algorithms converge to an optimal constrained policy that is both safe and performant.
While effective at enforcing safety, conventional safety filtering methods tend to be overly conservative in interactive scenarios, since they reason only in the physical space and do not account for the robot's ability to reduce uncertainty about other agents.
In particular, safety filtering in this case is \textit{open-loop} in the sense that the robot treats its uncertainty over other agents as static and assumes that it will not acquire new information to better understand other agents' behavior.
Alternatively, by proactively reasoning about the \emph{future information} they might receive,
robots can make more informed decisions and become less conservative.
Motivated by this idea, Bajcsy et al.~\cite{bajcsy2021analyzing} formulate a reachability problem in the joint space of physical states and estimated parameters of a human model to analyze how quickly the robot can learn such parameters in order to make safe, timely decisions.
Belief-space safety filters (\beliefsf)~\cite{hu2023deception} extends this idea to fully coupled multi-agent interaction scenarios where the robot is not merely an observer, but has the ability to actively engage with other agents to gain new information, leading to a \emph{closed-loop-information} safety analysis. 
In order to leverage richer and more complex observations, e.g., RGB images, recent works~\cite{nakamura2025generalizing,seo2025uncertainty,lutkus2025latent} develop latent-space safety filters that reason about safety with a learned world model or reduced order model.

We develop, for the first time, a safety verification framework for neural approximations of belief-space safety filters.
Our analysis also suggests a path toward provable guarantees for emerging safety filters that operate in learned latent spaces.

\p{Safety Verification and Conformal Prediction}
While it is generally challenging to synthesize learning-based control policies that are provably safe by design, safety verification has emerged as a practical way to certify the safety of neural controllers prior to their deployment.
Existing safety verification frameworks for autonomous systems predominantly focus on settings with known dynamic models using various control-theoretic tools, \eg, reachable tubes~\cite{hu2020reachsdp,everett2021reachability}, barrier certificates~\cite{prajna2004safety}, hybrid system analysis~\cite{ivanov2019verisig}, and mixed-integer optimization~\cite{ahn2017safety}.
Conformal prediction (CP)~\cite{angelopoulos2023conformal} has become a popular tool for providing safety guarantees for complex, learning-enabled dynamical systems since it is distribution-free and does not require knowledge of the underlying system dynamics.
Lindemann et al.~\cite{lindemann2023conformal} use CP to calibrate prediction regions of other agents' future trajectories, which are incorporated in a downstream model predictive controller for safe planning.
Dixit et al.~\cite{dixit2023adaptive} extends this approach to account for distribution shift using adaptive CP.
More recently, advanced CP variants~\cite{mirzaeedodangeh2025safe,binny2025moved} have been proposed to explicitly handle interactions in multi-agent scenarios. 
Most closely related to our work are CP algorithms for certifying neural HJ-based safe sets and controllers.
Most notably, Lin and Bansal~\cite{lin2024verification} leverage split CP to provide high-probability safety guarantees for systems using neural HJ safe controllers.
They also show that such a CP formulation reduces to a scenario optimization problem.

While Lin and Bansal's algorithm focuses on single-system safety in the physical space, our work builds on it to certify neural safety filters that operate in the joint physical--belief space for interactive robotics under uncertainty.
We show that focusing verification on regions where inference is expected to be reliable can yield a substantially more permissive belief-space safety filter than an inference-agnostic CP baseline.

\section{Problem Setup and Background}
\label{sec:problem_setup}

\p{Problem Setup} Consider an ``ego'' autonomous robot ($\ego$) interacting with one or multiple ``opponent'' agents ($\oppo$), described by a dynamical system
\begin{equation}
\label{eq:dyn_sys}
    \state_{t+1} = \dyn (\state_t, \ctrl^\ego_t, \ctrl^\oppo_t),
\end{equation}
where $\state_t = (\state^\ego_t, \state^\oppo_t) \in \statespace\subseteq\reals^{\nx}$ is the joint state vector of the ego and the opponent, $\ctrl^\ego_t \in \cset^\ego \subset \reals^{\me}$ and $\ctrl^\oppo_t \in \cset^\oppo \subset \reals^{\mo}$ are the control vectors of the robot and the opponent, respectively, and
$f: \statespace \times \cset^\ego \times \cset^\oppo \to \statespace$ describes the physical interactions.
We capture catastrophic safety events (\eg, collisions) with the \emph{failure set}:
\begin{equation}
\failureset\coloneq\{\state\in\xset\mid g(\state)<0\},
\end{equation}
where $g:\xset\rightarrow\reals$ is a safety \textit{margin} function quantifying how close state $\state$ is to violating safety.

In this paper, we are interested in verifying that a learned \emph{safe set} $\safeset$ (assuming $\safeset \cap \failureset = \emptyset$) and a corresponding control policy can guarantee safe interactions between the ego robot and the opponent with high probability within a finite horizon $T$, if the system starts within $\safeset$: 
\begin{equation}
    \label{eq:overall_safety}
    \prob_{\state_0 \in \safeset} (\state_t \in \failureset) \leq \epsilon,~\forall t = 0, 1, \ldots, T.
\end{equation}

\begin{remark}
    Alternatively, one may also verify a stronger argument than~\eqref{eq:overall_safety} that safe set $\safeset$ is controlled-invariant (\ie, the state $\state$ does not leave $\Omega$ under a specific control policy): $\prob_{\state_0 \in \safeset} (\state_t \in \safeset^\compl) \leq \epsilon,~\forall t = 0, 1, \ldots, T$.
\end{remark}

\p{Safety Filters, Maximal Safe Set, and Hamilton--Jacobi Reachability}
In order to achieve safety objective~\eqref{eq:overall_safety} while minimally interfering with the ego robot's task performance, we consider \textit{safety filter}, an automatic process that monitors the operation of the ego robot at runtime and intervenes, when deemed necessary, by modifying its task policy $\policyTask(\state_t)$ with a filtered action $\ctrl_t = \safetyFilter(\state_t, \policyTask)$ to prevent a potential safety failure.
Here, the function form of $\safetyFilter$ depends on the intervention type of a safety filter, which includes, \eg, switching, transition, and optimization; our method is designed to work for generic value-based safety filters~\cite[Sec.~3.2]{hsu2023safety} and therefore agnostic to the specific type of $\safetyFilter$ as long as a safety value is available for the filter's monitoring.

Ideally, we want the most \emph{permissive} safety filter that works within the \emph{maximal} safe set $\safeSetMax$, which contains all states in $\statespace$ where recursive safety is guaranteed despite the worst-case opponent action.
In theory, such $\safeSetMax$ and the corresponding safe controller can be computed with game-theoretic \gls{HJ} reachability analysis~\cite{mitchell2005time,hsunguyen2023isaacs,wang2024magics}.
HJ analysis captures the opponent's worst-case behavior with an infinite-horizon zero-sum dynamic game.
In discrete time, the game's solution can be obtained via the fixed-point Isaacs equation~\cite{isaacs1954differential} (the game-theoretic counterpart to the Bellman equation):
\begin{equation}
\label{eq:Isaacs_basic}
\begin{aligned}
    \valfunc(\state) = \min \left\{ g(\state), \max_{\ctrl^\ego \in \cset^\ego} \min_{\ctrl^\oppo \in \cset^\oppo} \valfunc \big(\dyn (\state, \ctrl^\ego, \ctrl^\oppo)\big)  \right\}.
\end{aligned}
\end{equation}
The \textit{safety value function} $\valfunc(\cdot)$ encodes the maximal safe set with its zero superlevel set, \ie, $\safeSetMax := \{\state \mid \valfunc(\state) \geq 0\}$.
Given a value function $\valfunc(\cdot)$, the optimal ego's defense and opponent's attack policies $(\policy^{\shield}, \policy^{\oppo,*})$ are obtained by taking argmax/argmin controls in~\eqref{eq:Isaacs_basic}.
Existing work has proposed several ways of using \gls{HJ} policies in safety filter $\safetyFilter$, including switching~\cite{mitchell2005time}, rollouts~\cite{nguyenhsu2024gameplay}, and smooth blending~\cite{borquez2024safety,oh2025safety}, each leading to a different intervention type and practical tradeoff between safe set volume $\safeset$ and the robot's task performance.

Despite its theoretical soundness, traditional \gls{HJ} analysis relies on grid-based dynamic programming, making it intractable to scale to high-dimensional state spaces.
Recent work has used data-driven methods to scale up the computation of \gls{HJ}-based safety filter synthesis by approximating value function $\valfunc$ or safe control policy $\policy^{\shield}$ with neural networks learned through, for example, deep reinforcement learning~\cite{fisac2019bridging,hsunguyen2023isaacs,wang2024magics,li2025certifiable} and self-supervised learning~\cite{bansal2021deepreach}.
Although such \textit{neural safety filters} have shown practical scalability and robustness in safe control for complex robotic systems, they do not come with any safety guarantees on the closed-loop system.
Recently, Lin and Bansal~\cite{lin2024verification} have developed a verification algorithm to certify high-probability safety guarantees~\eqref{eq:overall_safety} for neural \gls{HJ}-based safety filters using conformal prediction~\cite{angelopoulos2023conformal}, which we briefly introduce in the next section.

\p{Verifying Neural Safety Filters with Conformal Prediction}
For a \emph{single-agent} system $\ego$ (\ie, in the absence of opponent $\oppo$) equipped with a neural safety value function $\valfunc$ and control policy $\policy^\shield$, Lin and Bansal~\cite{lin2024verification} seek to verify that recursive safety holds with high probability for $\safeset_\delta := \{\state \mid \valfunc(\state) \geq \delta\}$, the $\delta$-superlevel set of $\valfunc$, within a time horizon $[0,T]$.
Given a robot trajectory $\traj_{\state_0} = (\state_0, \state_1, \ldots, \state_T)$ resulted from applying policy $\policy$, define its \emph{empirical safety score} as:
\begin{equation}
    \label{eq:em_safety_core}
    \cost_{\policy} (\traj_{\state_0}) =  \min_{t\in[0,T]} g(\state_t), 
\end{equation}
which is non-negative if and only if the ego robot always stays outside failure set $\failureset$ for all $t \in [0,T]$.
The verification proceeds as follows.
First, $N$ independent and identically distributed (i.i.d.) initial states $\state^{1:N}_0$ are sampled within $\safeset_\delta$ according to some probability distribution $\prob$ over $\safeset_\delta$.
Since $\safeset_\delta$ typically does not have a closed form, rejection sampling is commonly used in this step.
Next, safety scores $\cost_{\tpolicy^\shield}(\traj_{\state_i})$ are computed for $i = 1,2,\ldots,N$ by forward simulating the system from $\state^i_0$ under $\tpolicy^\shield$.
Finally, we count $k$ empirically unsafe trajectories, where $J_{\tpolicy^\shield} < 0$.
The following proposition provides a probabilistic safety guarantee of the closed-loop system operating within $\safeset_\delta$.
\begin{proposition}[Lemma 5~\cite{lin2024verification}]
\label{prop:lin_cp}
    If a safety coverage (violation) parameter $\epsilon \in (0,1)$ and a confidence parameter $\beta \in (0,1)$ satisfy
    $
    \sum_{i=0}^{k} {\binom{N}{i}} \,\epsilon^{\,i} (1-\epsilon)^{\,N-i} \le \beta
    $,
    then, with probability at least $1-\beta$,
    $\prob_{\state \in \safeSet_\delta}\!\left(\cost_{\tpolicy^\shield}(\traj_\state) \le 0\right) \le \epsilon$.
\end{proposition}

The proof of \autoref{prop:lin_cp} is based on standard split \gls{CP}~\cite{angelopoulos2023conformal} and can be found in~\cite{lin2024verification}.
In this paper, we extend and apply Lin and Bansal's method to the verification of a recently proposed belief-space safety filter~\cite{hu2023deception} that safeguards against a potentially adversarial opponent while adapting to uncertainty during the interaction. We introduce this safety filter in the next section.

\p{\beliefsf: Reducing Conservativeness by Filtering in Information Space}
Standard neural approximations to the two-player safety value function in~\eqref{eq:Isaacs_basic}, \eg, the ISAACS framework~\cite{hsunguyen2023isaacs}, assume that the opponent may take any action from its control set $\cset^\oppo$ at all times.
This assumption can lead to overly conservative solutions, such as excessively defensive robot behaviors that prevent efficient task execution, or a learned safe set $\safeset$ whose volume is significantly smaller compared to that of $\safeset^*$.

Belief-Space Safety Filter (\beliefsf)~\cite{hu2023deception} proposes to make the filtering \textit{more permissive} by conditioning the opponent's control bound on the robot's inference (runtime learning) algorithm,
which reduces the robot's uncertainty about the opponent at runtime.
For instance, in the vehicle--pedestrian interaction scenario described in the running example and visualized in \autoref{fig:re}, imagine that when the human begins to step onto the highway by crossing the road boundary, the robot will have made the corresponding observation, which shifts its belief towards the hypothesis that the human's true intent is to cross the road.
Assuming that the robot's belief about human intent is indeed accurate (we will revisit this assumption shortly), the safety filter can temporarily discard other low-probability hypotheses about the human's intent and subsequently ``shrink'' the opponent's control bound during the safety analysis, thus reducing the conservativeness of the filter.

\begin{figure}[!hbtp]
  \centering
  \includegraphics[width=1.0\columnwidth]{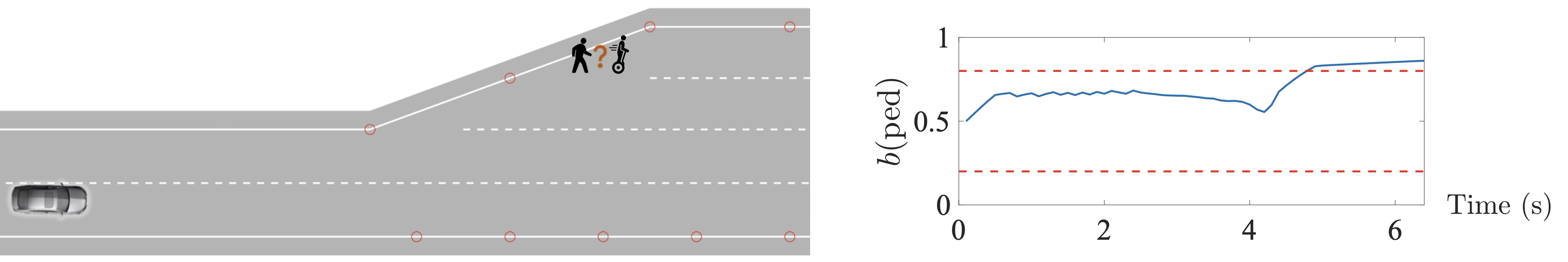}
  \caption{\label{fig:re} \beliefsf~applied to vehicle--pedestrian interaction (running example). \emph{Left:} The interaction scene, where an ego autonomous vehicle is uncertain about the semantic class (pedestrian or Segway rider) and intended destination (red dots) of an opponent human crossing the road.
  \emph{Right:} A representative closed-loop belief trajectory of $\belstate(\theta_\type=\ped)$, where the dashed red lines represent threshold $\epsilon_{\theta_\type} = 0.2$.
  The adversarial opponent learned to exploit the ego robot by generating motions that keep the robot's belief close to uniform, thereby confusing its inference module and making it challenging to timely identify the opponent's hidden type $\theta$ and the corresponding control bound $\cset^\oppo_\theta(\state_t)$.
  }
\end{figure}

Concretely, in \beliefsf, we model the opponent's \emph{hidden type}~\cite{kajii1997robustness} with a discrete variable $\theta\in\Theta$, which may encode its individual characteristics that are unknown to the ego, \eg, its intended destination, semantic class, and willingness to cooperate, etc.
The ego robot then represents its uncertainty over the unknown type $\theta$ through a \textit{belief} state $\belstate(\theta) \in \belspace$.
Using an inference (runtime learning) algorithm, \eg, Bayesian inference and neural motion predictors, the robot can update (and hopefully refine) its belief over time as new observations are made. 
We can combine physical dynamics~\eqref{eq:dyn_sys} and this runtime inference process to define a joint dynamical system:
\begin{equation}
\label{eq:jnt_dyn}
\jstate_{t+1} = \jntdyn(\jstate_t, \ctrl^\ego_{t}, \ctrl^\oppo_{t}) :=
\begin{bmatrix}
\dyn (\state_t, \ctrl^\ego_t, \ctrl^\oppo_t) \\[0.1cm]
\beldyn({\belstate}_{t}, \obs_t)
\end{bmatrix},
\end{equation}
where $\jstate_t := (\state_t, \belstate_t) \in \jsset$ is an \emph{information} (\ie, joint physical--belief) state vector, $\obs_t := \sigma (\state_{t}, \ctrl^\oppo_{t}) \in \obsspace$ is the robot's current observation with an observation map $\sigma$ modeling partial observability, and $\beldyn: \belspace \times \obsspace \rightarrow \belspace$ is the dynamic model of the robot's runtime inference process.
Joint dynamics~\eqref{eq:jnt_dyn} open up the opportunities of safety analysis and filtering in \textit{information space}. 
We can now extend the standard Isaacs equation~\eqref{eq:Isaacs_basic} to the information state $\jstate$:
\begin{equation}
\label{eq:Isaacs_belief}
\begin{aligned}
    \valfunc(\jstate) = \min \left\{ g(\state),  \max_{\vphantom{\csetest}\ctrl^\ego \in \cset^\ego} \min_{\ctrl^\oppo \in {\csetest}^\oppo(\jstate)} \valfunc \left(\jntdyn(\jstate, \ctrl^\ego, \ctrl^\oppo)\right)  \right\}.
\end{aligned}
\end{equation}
Here, ${\csetest}^\oppo(\jstate)$ is the ego's \emph{inferred control bound} of the opponent---a key identity that bridges runtime inference and closed-loop safety---defined as
\begin{equation}
\label{eq:pred_ctrl_bd}
    \csetest^\oppo (\jstate_t) := 
    \bigcup_{\theta \in \Theta} \csetest^{\oppo}_\theta (\jstate_t), \quad 
    \csetest^{\oppo}_\theta (\jstate_t) :=
    \begin{cases}
    \cset^\oppo_\theta(\state_t),  & \text{if } \belstate_t(\theta) \geq \epsilon_\theta \\
    \emptyset,  & \text{if } \belstate_t(\theta) < \epsilon_\theta
    \end{cases} 
\end{equation}
where $\cset^\oppo_\theta(\state_t)$ is a known (possibly state-dependent) set of \emph{admissible} opponent actions associated with hypothesized type~$\theta$, and $\epsilon_\theta \geq 0$ is a designer-specified threshold.
Below, we use a running example to illustrate that \beliefsf~improves practical safety and performance compared to neural safety filters synthesized only in physical space.

\begin{runningexample}
\emph{Running Example (Vehicle--Pedestrian Interaction).}

\p{System} We consider an 18D system from~\cite{hu2023deception} describing an uncertain vehicle--pedestrian interaction scenario, as depicted in \autoref{fig:re}.
The autonomous vehicle (ego $\ego$) and crossing human (opponent $\oppo$) are modeled as a 5D kinematic bicycle model and a 2D particle with velocity control, respectively.
The information state vector is structured as $\jstate := (\state^\ego, \state^\oppo, \belstate(\theta_\type),\belstate(\theta_\goal)) \in \reals^{18}$, where $\theta_\type \in \{\ped,\seg\}$ models the ego's hypotheses about the opponent's semantic class (pedestrian or Segway rider), with associated control sets 
$\cset^\oppo_\ped = 
[-0.75,0.75]
\times
[-2,2]\ms$ and $\cset^\oppo_\seg = [-0.75,0.75]
\times
[-8,0]\ms$, and $\theta_\goal \in \{g_i\}_{i=1}^{9}$ models the hypotheses about the human's goal position, scattered along the road boundary (red dots in \autoref{fig:re}), with associated control set $\cset^\oppo_{g_i}$ defined as a box centered around a goal-reaching nominal control action for $g_i$.
The robot uses Bayesian inference as the learning dynamics $\beldyn$: $\belstate_{t+1}(\theta) \propto \prob(\ctrl^\oppo_t \mid \state_t; \theta)\belstate_t(\theta)$. The likelihood function is a Gaussian $\prob(\ctrl^\oppo_t \mid \state_t; \theta) = \mathcal{N}(\mu_\theta, \Sigma_\theta)$, where $\mu_\theta$ is the average control of the bound $\cset^\oppo_\theta(\state_t)$ and $\Sigma_\theta$ is a predefined covariance matrix.

\p{Safety Specifications}
A trajectory is considered safe if (1) no collision occurs between the ego vehicle and the human, and (2) the ego vehicle remains within the road boundary.
We define the failure set $\failureset$ using a safety margin function $g(\state)$ that encodes both collision avoidance and road-boundary constraints, where $g(\state) < 0$ indicates that at least one safety constraint is violated.

\p{\beliefsf~Synthesis and Deployment}
We learned a \beliefsf~$\safetyFilter$ using a game-theoretic adversarial 
\gls{RL} approach~\cite{hsunguyen2023isaacs,wang2024magics} that approximately solves a time-discounted variant of Isaacs equation~\eqref{eq:Isaacs_belief}.
The result includes a neural value function $\valfunc$ and safety fallback policy $\policy^\shield$.
The training details can be found in~\cite{hu2023deception}.
At runtime, we deploy $\safetyFilter$ as a value-based, switch-type~\cite[Sec.~3.1]{hsu2023safety} safety filter:
\begin{equation*}
    \ctrl^\ego = \safetyFilter \left(\jstate, \policyTask(\jstate) \right) := %
    \begin{cases}
         \policyTask(\jstate), \quad & \valfunc \Big( \jntdyn \left( \jstate, \policyTask(\jstate), \ctrl^\oppo \right) \Big) \geq \delta, \\
         \policy^\shield(\jstate), & \text{otherwise},
    \end{cases}
\end{equation*}
where the safety fallback policy only overrides the task policy to maintain non-negativity (hence safety) of $\valfunc(\cdot)$ when state $\jstate$ approaches the boundary of the learned safe set $\safeset_\delta$.

\p{Results}
We ran 1000 randomized trials with the opponent applying an adversarial policy, also synthesized with game-theoretic RL.
The \beliefsf~with $\delta=0$ achieved an empirical safe rate of $94.3~\%$ and a task completion rate of $99.8~\%$ (a trial is considered timeout if $t>10$ s).
In contrast, the baseline safety filter that reasons only in the physical state space, synthesized with the same game-theoretic RL algorithm, yielded an empirical safe rate of $77.7~\%$ and a task completion rate of only $37.6~\%$.
\textit{These results demonstrate that \beliefsf~is significantly safer and more performant} in practice than standard neural \gls{HJ}-based safety filters operating purely in physical spaces.
It is also worth noting that, beyond physical interactions, the opponent in this example learns to exploit the ego agent in the belief space through \emph{deceptive behaviors} that hinder the ego's inference by keeping its belief close to uniform, as evidenced by the belief trajectory in \autoref{fig:re}.
This observation underscores the importance of providing formal safety guarantees for \beliefsf~in scenarios involving strategic, potentially adversarial opponents that may deliberately exploit the ego robot's runtime inference.
\end{runningexample}

Now, one may ask under what conditions a \beliefsf~$\safetyFilter$ can \emph{guarantee} safety for system~\eqref{eq:dyn_sys}.
It has been shown in~\cite{hu2023deception} that closed-loop safety can be ensured under the \emph{Inference Hypothesis} (\autoref{def:IH}), that is, given an optimal value function $\valfunc$ according to \eqref{eq:Isaacs_belief}, if the inferred opponent control bound $\csetest^\oppo (\jstate_t)$ predicted by $\beldyn$ can always include the opponent's executed control action $\ctrl^\oppo_t$, then recursive safety is guaranteed when the ego applies $\beliefsf$ against \emph{any} opponent policy.
However, such assumption may be too strong to hold in practice for two reasons: (1) due to the high dimensionality of information spaces\footnote{The joint system $\jntdyn$ in the running example is 18D; grid-based \gls{HJ} analysis cannot practically scale beyond 7D.}, the value function $\valfunc$ and associated safe policy $\policy^\shield$ need to be approximated using, \eg, neural networks, which inevitably introduces error, and (2) the opponent action is likely outside the inferred bound $\csetest^\oppo (\jstate_t)$ due to, \eg, delayed inference or poorly designed parameters ($\cset^\oppo_\theta(\state_t)$ and $\epsilon_\theta$ in \eqref{eq:pred_ctrl_bd}), thus invalidating the Inference Hypothesis.
In this paper, we develop a practical algorithm to guarantee probabilistic safety for \beliefsf~without relying on the Inference Hypothesis.

\begin{definition}[Inference Hypothesis~\cite{hu2023deception}]
\label{def:IH}
    At each time $t$, the ego's belief $\belstate_t$ must have assigned no less than $\epsilon_\theta$ probability to at least one $\theta$ consistent with opponent action $\ctrl^\oppo_t$. This is equivalent to requiring $\ctrl^\oppo_t \in \csetest^\oppo (\jstate_t),~\forall t \geq 0$, with $\csetest^\oppo (\jstate_t)$ defined in~\eqref{eq:pred_ctrl_bd}.
\end{definition}

\section{Permissive and Assured Belief-Space Safety Filtering}
\label{sec:approach}

\p{Problem Statement: Permissive Safety Verification of \beliefsf}
In this paper, our main objective is to develop a verification algorithm for \beliefsf, to certify that its safety guarantees hold with high probability within a learned belief-space safe set.
We further use insights gained from verification to modify the deployment of belief-space safety filters.
In particular, we show that verifying and deploying \beliefsf~in the \emph{trusted inference region}---a subset of the learned safe set where the robot's inference algorithm is expected to perform reliably---is key to reducing the practical conservativeness of belief-space safety filtering.

\subsection{Conformal Probabilistic Safety Verification for \beliefsf}
\label{sec:approach:direct_cp}

We start by considering a more direct extension of Lin and Bansal's verification algorithm~\cite{lin2024verification} to the regime of belief-space safety.
Define $\distr$ as an unknown distribution over the trajectories of information states $\jtraj = (\jstate_0, \jstate_1, \ldots, \jstate_T)$.
Next, we make a standard assumption in conformal prediction for safe multi-agent interaction, \eg, MPC-based social navigation~\cite{lindemann2023conformal}.

\begin{assumption}
\label{assump:ds}
    The control inputs of the ego robot and opponent $(\ctrl^\ego_t,\ctrl^\oppo_t)$ do not change the distribution $\distr$ for all $t \in [0,T]$.
\end{assumption}

\begin{remark}
Assumption~\ref{assump:ds} is a standard assumption for ensuring exchangeability commonly used in conformal prediction for interactive systems~\cite{lindemann2023conformal}.
Intuitively, it assumes that applying the safety filter does not alter the underlying distribution of interaction trajectories.
When the robot's filtered actions substantially influence the opponent's behavior and alter the interaction distribution, this assumption may no longer hold.
In such cases, recent CP methods that explicitly account for interaction-induced distribution shift~\cite{mirzaeedodangeh2025safe,binny2025moved} provide a promising direction for extending our framework.
\end{remark}

Given an ego robot's task policy $\policy^{\task,\ego}$, opponent's policy $\policy^\oppo$, and a \beliefsf~$\safetyFilter$ defined by neural value function $\valfunc$ and safe control policy $\policy^\shield$, our goal is to verify that within the $\delta$-superlevel set $\safeset_\delta := \{\jstate \mid \valfunc(\jstate) \geq \delta\}$,
system~\eqref{eq:jnt_dyn} in closed-loop with $\safetyFilter$ is recursively safe, \ie, $\state_t \notin \failureset,~\forall t \in [0,T]$, with high probability. 
The approach, which we refer to as \directcp, resembles that of Lin and Bansal~\cite{lin2024verification}, with the main difference that the sampling is now performed in the information space that encompasses both the ego and opponent's physical states as well as belief states.
We first sample $N$ i.i.d. initial states $\jstate^{1:N}_0 \in \safeset_\delta$ and obtain trajectories $\jtraj_{z^{1:N}}$ by forward simulating joint system $\jntdyn$ with the ego applying $\safetyFilter(\cdot, \policy^{\task,\ego})$ and the opponent applying policy $\policy^\oppo$.
Then, we compute empirical safety scores $\cost_{\safetyFilter}(\jtraj_{\jstate_i}) = \min_{t\in[0,T]} g(\state_t)$ for $i = 1,2,\ldots,N$ and count $k$ empirically unsafe trajectories where $J_{\safetyFilter}(\cdot) < 0$.

\begin{runningexample}
\emph{Running Example:} We applied \directcp~to verify \beliefsf~in the running example, where the opponent applies a fixed adversarial policy synthesized with game-theoretic RL.
The verification horizon was $T = 10$ s.
We used a confidence parameter $\beta = 10^{-5}$. 
With a safe set level $\delta = 0.1$, we obtained $k=438$ failure cases from $N=20000$ samples, leading to a certified safety level (coverage) $1-\epsilon = 97.34\%$.
Among the $20000$ trials, inference failed in $475$ cases, $338$ of which also resulted in unsafe outcomes.
The safety level is increased to $98.7\%$ with $\delta = 0.7$.
Additional results with varying safe set levels $\delta$ can be found in \autoref{fig:sr}.
\end{runningexample}

As reported in the above running example, the empirical safe rate when using \beliefsf~is significantly lower when the robot's inference algorithm has failed, \ie, when it does not account for the opponent's executed action $\ctrl^\oppo_t$ in the inferred bound $\csetest^\oppo (\jstate_t)$.
This phenomenon is expected since the synthesis of \beliefsf~via Isaacs equation~\eqref{eq:Isaacs_belief} assumes that the robot's inference is accurate and that the opponent's control always lies within the predicted control bound.
Motivated by this observation, in the next section, we propose an improved \gls{CP} procedure by explicitly accounting for the robot's inference performance in belief-space safety filtering, so that the verification algorithm can return a tighter (\ie, smaller) safety coverage parameter $\epsilon$ without needing to shrink the safe set volume with an excessively large $\delta$, thus reducing the conservativeness of safety filtering.

\begin{remark}
    Although the opponent policy used in verification is an approximation of a worst-case adversarial policy, the safety guarantees provided by conformal prediction are \emph{distributional}: they hold with respect to the trajectory distribution induced by this fixed opponent policy, rather than a worst-case guarantee over all admissible opponents.
    Nevertheless, verifying against a fixed, highly exploiting opponent remains useful in practice.
    Empirically, we often observe that such opponents tend to induce safety-critical interactions that stress both the safety filter and the runtime inference module~\cite{hu2023deception,wang2024magics,nguyenhsu2024gameplay}.
    When the ego policy is paired with a less aggressive opponent, the resulting trajectory is typically no more safety-critical than those seen during verification, so the empirical failure rate under the same ego policy is expected to be no higher (and often lower) than what is observed during verification.
\end{remark}

\subsection{Verifying and Deploying \beliefsf~in Trusted Inference Regions}

Our key insight is that, when the robot's inference algorithm is sufficiently reliable, \ie, it can successfully contain $\ctrl^\oppo_t$ within $\csetest^\oppo (\jstate_t)$ for the majority of state space with high probability, we may choose to identify and ``carve out'' the small subset in state space where the inference algorithm can fail. This way, we can still expect to recover a large portion of the safe set.
To formalize this idea, we start by introducing the \textit{inference score} $\costinf = h_L(\jstate)$, where $h_L: \jsset \rightarrow \reals$ is a user-specified function mapping an information state $\jstate$ to a score $\costinf$ that accesses the \textit{quality} of the robot's inference algorithm; the inference algorithm is \textit{believed} to be accurate for physical state $\state$ in $\jstate$ if $h_L(\jstate) \geq 0$.
In the context of \beliefsf, a reasonable design of the inference score function can be as simple as a classifier that outputs $h_L(\jstate) = 1$ if it \textit{predicts} for the information state trajectory with $\jstate_0 = \jstate$ that $\ctrl^\oppo_t \in \csetest^\oppo (\jstate_t),~\forall t \geq 0$, and $h_L(\jstate) = -1$ otherwise; alternatively, $h_L$ can be a function of the predicted probability or confidence score\footnote{For example, if the model's predicted probability or confidence of belief trajectory $\beltraj=(\belstate_0,\belstate_1,\ldots)$ is too low, $h_L$ can return a negative number to predict poor inference performance.}, typically available (or easy to calculate) from the output of modern multi-modal trajectory prediction models~\cite{shi2022mtr}.
Note that the inference score function $h_L$ is purely an evaluation on the quality of the inference algorithm $\beldyn$, and should therefore be chosen \textit{independently} of the safety filter $\safetyFilter$.
We may now define the notion of \emph{trusted inference region}, a subset of the learned safe set where we also expect the robot's inference algorithm to be accurate.

\begin{definition}[Trusted Inference Region]
\label{def:TIR}
    Given a score function $h_L$ and learned information-space safe set $\safeset_\delta$, the trusted inference region is defined as $\safeset_{\delta | h} := \{\jstate \mid \valfunc(\jstate) \geq \delta, h_L(\jstate) \geq 0 \}$.
\end{definition}

\begin{runningexample}
\emph{Running Example:} We trained a classifier $h_L(\jstate)$ that outputs $1$ if it predicts that the information state trajectory satisfies $\ctrl^\oppo_t \in \csetest^\oppo (\jstate_t),~\forall t \in [0,T]$, and $-1$ otherwise.
The classifier was trained on a dataset of $10^6$ randomly sampled initial information states $\jstate_0$ with labels obtained by rolling out joint dynamics $\jntdyn$.
On a test dataset of size $2 \times 10 ^5$, $h_L(\jstate)$ leads to a false inference rate (\ie, rejection rate) of $4.67\%$ with a misclassification rate of $0.92 \%$.
\end{runningexample}

Now, we are ready to introduce Joint Inference--Safety Test (\jist), an inference-aware \gls{CP}-based safety verification algorithm tailored for \beliefsf.
Our goal is to verify that system trajectories starting within $\safeset_{\delta | h}$ under \beliefsf~$\safetyFilter$ is safe with high probability.
Similar to \directcp, we begin by sampling $N$ i.i.d. initial states $\jstate^{1:N}_0$ within the trusted inference region $\safeset_{\delta | h}$.
We then forward-simulate $\jntdyn$ under $\safetyFilter$ and $\policy^\oppo$ to obtain trajectories $\jtraj_{z^{1:N}}$.
Since, ultimately, we care about safety assurances under $\safetyFilter$,
we only need to compute the empirical safety score $\cost_{\safetyFilter}(\jtraj_{\jstate_i}) = \min_{t\in[0,T]} g(\state_t)$ for $i = 1,2,\ldots,N$ and count $k$ failure cases, the same as in \directcp.
The above \jist~verification procedure is summarized in Algorithm~\ref{alg:offline_cp}.
Typically, we want the verification result to hold with high confidence defined by a target value $\bar{\beta}$ that is sufficiently small.
In the event that $\beta > \bar{\beta}$, one may rerun Algorithm~\ref{alg:offline_cp} by increasing $\epsilon$, $\delta$, or $N$, thereby achieving the desired confidence at the expense of increased safety conservativeness or sample complexity.
Compared to \directcp, \jist~does not introduce any additional sample or computation complexity, since it only differs by focusing sampling in the trusted inference region $\safeset_{\delta | h}$.
In Section~\ref{sec:case_study}, we will see that this minimal change can lead to a significantly reduced conservativeness for belief-space safety filtering in practice.
The following theorem provides a probabilistic safety guarantee of the closed-loop system $\jntdyn$ operating within the trusted inference region $\safeset_{\delta | h}$.

\begin{theorem}[Safety Verification of \beliefsf~with \jist]
\label{thm:jist}
    Given a trusted inference region $\safeset_{\delta | h}$ defined in \autoref{def:TIR}, if a safety coverage parameter $\epsilon \in (0,1)$ and a confidence parameter $\beta \in (0,1)$ satisfy
    $
    \sum_{i=0}^{k} {\binom{N}{i}} \,\epsilon^{\,i} (1-\epsilon)^{\,N-i} \le \beta
    $,
    then 
    $\prob_{\jstate \in \safeSet_{\delta|h}}\!\left(\cost_{\safetyFilter}(\jtraj_\jstate) \le 0\right) \le \epsilon$ with probability at least $1-\beta$.
\end{theorem}

\begin{proof}
The proof follows directly from standard split \gls{CP} results and resembles that of~\cite[Theorem~3]{lin2024verification}.
We provide a brief sketch here for completeness.
Define the nonconformity score as the negative safety score $-\cost_{\safetyFilter}(\jtraj_\jstate)$ and set $\alpha = \frac{k+1}{N+1}$.
By~\cite[Theorem~1]{angelopoulos2023conformal}, we have 
$\mathbb{P}_{\jstate \in \safeset_{\delta | h}}\!\left(\cost_{\safetyFilter}(\jtraj_\jstate)\ge -\hat{q}\right)\ge 1-\alpha$, where quantile $\hat{q}$ is the $(N-k)$-th safety score, \ie, the largest negative score of $\cost_{\safetyFilter}(\cdot)$.
Therefore, we have that $\mathbb{P}_{\jstate \in \safeset_{\delta | h}}\!\left(\cost_{\safetyFilter}(\jtraj_\jstate)> 0\right)\ge \frac{N-k}{N+1}$.
In addition, by~\cite[Sec.~3.2]{angelopoulos2023conformal}, we can further show that the distribution of safety coverage admits a closed-form Beta distribution, \ie,
$\mathbb{P}_{\jstate \in \safeset_{\delta | h}}\!\left(\cost_{\safetyFilter}(\jtraj_\jstate)> 0\right) \sim \operatorname{Beta}(N-k,k+1)$.
Equivalently, we have
$
\mathbb{P}\!\left(p \le \epsilon\right)
= \mathbb{P}\!\left(1-p \ge 1-\epsilon\right)
= 1 - F_{\mathrm{Beta}(N-k,k+1)}(1-\epsilon)
= 1- \sum_{i=0}^{k} {\binom{N}{i}} \,\epsilon^{\,i} (1-\epsilon)^{\,N-i}
$, where $
p \;:=\; \mathbb{P}_{\jstate \in \safeset_{\delta | h}}\!\left(\cost_{\safetyFilter}(\jtraj_\jstate)\le 0\right)
$ and $F_{\mathrm{Beta}(\cdot,\cdot)}$ denotes the Beta CDF.
Hence, if $(\epsilon,\beta)$ satisfy
$
\sum_{i=0}^{k} {N \choose i}\epsilon^{i}(1-\epsilon)^{N-i} \le \beta,
$
then $\mathbb{P}(p \le \epsilon)\ge 1-\beta$, i.e., with probability at least $1-\beta$ we have
$
\mathbb{P}_{\jstate \in \safeset_{\delta | h}}\!\left(\cost_{\safetyFilter}(\jtraj_\jstate)\le 0\right)\le \epsilon.
$
\qed
\end{proof}

Once a trusted inference region $\safeset_{\delta | h}$ and \beliefsf~$\safetyFilter$ are verified with Algorithm~\ref{alg:offline_cp}, it is straightforward to deploy them online for interactive motion planning with safety guarantees, as shown in \autoref{fig:front_fig}.
In essence, $\safeset_{\delta | h}$ is used as an additional safety check to safety filter $\safetyFilter$'s own monitoring scheme.
It is only at the beginning of the interaction that we check if the initial information state $\jstate_0$ is within $\safeset_{\delta | h}$.
If so, we may apply belief-space safety filtering in the standard way as in~\cite{hu2023deception}, and the resulting trajectory is provably safe according to the verification results.
Otherwise, the system raises a flag and executes either \textsc{Reject()} or \textsc{Fallback()}.
The \textsc{Reject()} option prevents the system from applying policies filtered under $\safetyFilter$, since they may not be safe; for example, in assisted driving, this amounts to handing control back to the human driver.
The \textsc{Fallback()} option attempts to replace $\safetyFilter$ with a backup filter, if available, that is more conservative but certified safe at the given state $\jstate_0$.
The procedure of deploying a verified \beliefsf~is summarized in Algorithm \ref{alg:online}.
Note that Algorithm \ref{alg:online} may also be used in a receding horizon fashion by recursively checking if $\jstate_t \in \safeset_{\delta | h}$, which is useful when the actual control horizon is longer than the verification horizon $T$.

\begin{remark}
    Since \jist~couples safety verification with the quality of inference, it is expected to outperform \directcp~in settings where the cause of safety violations is dominated by incorrect inference, that is, the robot's runtime inference and neural safety filter are sufficiently accurate so that the ``correct-inference'' regime yields a substantially higher empirical safe rate.
    As both runtime inference (\eg, trajectory prediction models) and neural safety filter synthesis methods continue to improve in accuracy, \jist~can directly benefit from these advances by expanding the certified safe set without requiring additional verification data.
\end{remark}

\begin{remark}
    Independent of the safety filter, we can optionally run a conformal prediction procedure to quantify the reliability of the inference module.
    This evaluation can help inform the choice between \jist~and \directcp~by checking whether inference errors are sufficiently rare to warrant coupling safety with inference during verification.
\end{remark}


\begin{algorithm}[!ht]
\DontPrintSemicolon
\small
\caption{Offline CP-based Safety Verification with \jist}\label{alg:offline_cp}
\SetKwInput{KwRequire}{Input}
\KwRequire{Violation probability $\epsilon>0$, safe set level $\delta \geq 0$, inference score function $h_L$, learned belief-space value function $\valfunc$ and safety filter $\safetyFilter$, ego's task policy $\policy^{\task,\ego}$, opponent policy $\policy^\oppo$, sample size $N>0$, horizon $T>0$}
\SetKwInput{KwRequire}{Output}
\KwRequire{Verified trusted inference region $\safeset_{\delta | h}$ with confidence parameter $\beta$}
$k \gets 0$\;
\For{$i=0,1,\ldots,N-1$}{
    Sample information state $\jstate^i_0$ from $\safeset_{\delta | h}$\;
    Obtain trajectory $\jtraj_{\jstate^i} = (\jstate^i_0,\jstate^i_1,\ldots,\jstate^i_T)$ by forward simulating joint system $\jntdyn$ with $\safetyFilter$, $\policy^{\task,\ego}$, and $\policy^\oppo$\;
    Compute empirical safety score $\cost_{\safetyFilter}(\jstate_i) = \min_{t\in[0,T]} g(\state_t)$\;
    \If{$\cost_{\safetyFilter}(\jstate_i) < 0$}{
    $k \gets k + 1$
    }
}
Compute a confidence parameter $\beta$ such that $\sum_{i=0}^{k} {\binom{N}{i}} \,\epsilon^{\,i} (1-\epsilon)^{\,N-i} \leq \beta$\;
\end{algorithm}

\vspace{0em}

\begin{algorithm}[!ht]
\DontPrintSemicolon
\small
\caption{Safe Interactive Planning with Verified \beliefsf}\label{alg:online}
\SetKwInput{KwRequire}{Input}
\KwRequire{Trusted inference region $\safeset_{\delta | h}$, initial information state $\jstate_0$, learned belief-space safety filter $\safetyFilter$, ego's task policy $\policy^{\task,\ego}$, horizon $T>0$}
\SetKwInput{KwRequire}{Output}
\If{$\jstate_0 \in \safeset_{\delta | h}$}{
    \For{$t=0,1,\ldots,T$}{
    Execute filtered robot's control $\ctrl_t^\ego = \safetyFilter(\jstate_t,\policy^{\task,\ego})$\;
    Observe new state $\state_{t+1}$\;
    Update belief state: $\belstate_{t+1} \gets \beldyn({\belstate}_{t}, \obs_t)$
    }
}
\Else{\textsc{Reject()} or \textsc{Fallback()}}
\end{algorithm}

\vspace{0em}

\section{Case Study}
\label{sec:case_study}

\begin{figure}[!hbtp]
  \centering
  \includegraphics[width=1.0\columnwidth]{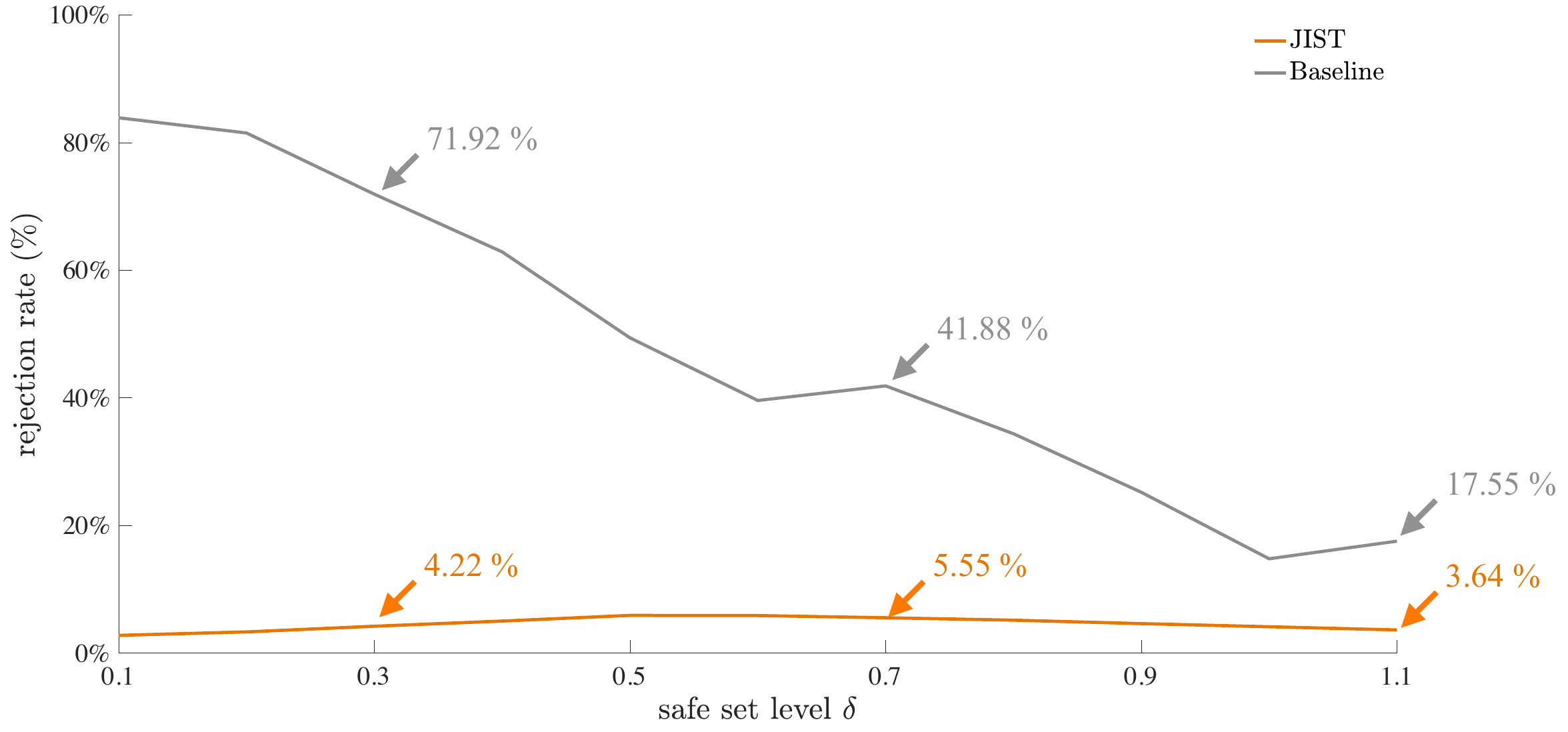}
  \caption{\label{fig:rr} Rejection rates of \beliefsf~computed from $20000$ randomized trials. For a fixed base safety level $\delta$ and safety coverage $\epsilon$, \jist~achieves a lower rejection rate than \directcp.
  }
  \vspace{-2em}
\end{figure}

In this section, we report verification and empirical results in the running example when the ego robot applies \jist.
For brevity, we refer readers to previous running example sections for details on the system setup, \beliefsf~synthesis and deployment, opponent modeling, and training of the inference score function $h_L(\cdot)$. 
Our goal is to evaluate whether coupling conformal verification with inference quality (\jist) can certify a more permissive belief-space safety filter than \directcp~while maintaining the same probabilistic safety guarantees.
Concretely, our evaluation focuses on two questions:
\begin{itemize}
\item for a fixed safe set level $\delta$ and calibration sample size $N$, whether \jist~yields a tighter safety coverage parameter $\epsilon$;
\item for a fixed safety coverage $\epsilon$ and calibration sample size $N$, whether \jist~certifies a more permissive filter, measured by rejection rate during online deployment.
\end{itemize}

For both methods, conformal prediction was performed using the same number of rollout samples ($N=20000$) and confidence parameter ($\beta = 10^{-5}$).
The inference score function $h_L(\cdot)$ used by \jist~was trained separately from, and independently of, the conformal prediction procedure, as described in the running example.
Importantly, these additional samples are only used to \textit{predict} the inference quality, and are not used by the conformal verification procedure itself; moreover, many contemporary inference modules, e.g., Transformer-based predictors, naturally provide confidence-related signals that can be used to construct such inference-quality scores.
With safe set level $\delta = 0.1$, we obtained $k=107$ failure cases from $N=20000$ trials with initial states sampled in the trusted inference region $\safeset_{\delta | h}$, leading to a certified safety level (coverage) $1-\epsilon = 99.21\%$, which is $1.87\%$ higher than that of \directcp.
Among the $107$ failed trials, inference was incorrect in $101$ cases ($94.4\%$).
We plot empirical and verified safety coverage of \jist~and \directcp~with varying safe set levels $\delta$ in \autoref{fig:sr}.
We observe that \jist~consistently outperforms \directcp~by $1$--$2\%$ in safety coverage for low to moderate values of the safe set level $\delta$.

\begin{figure}[!hbtp]
  \centering
  \includegraphics[width=1.0\columnwidth]{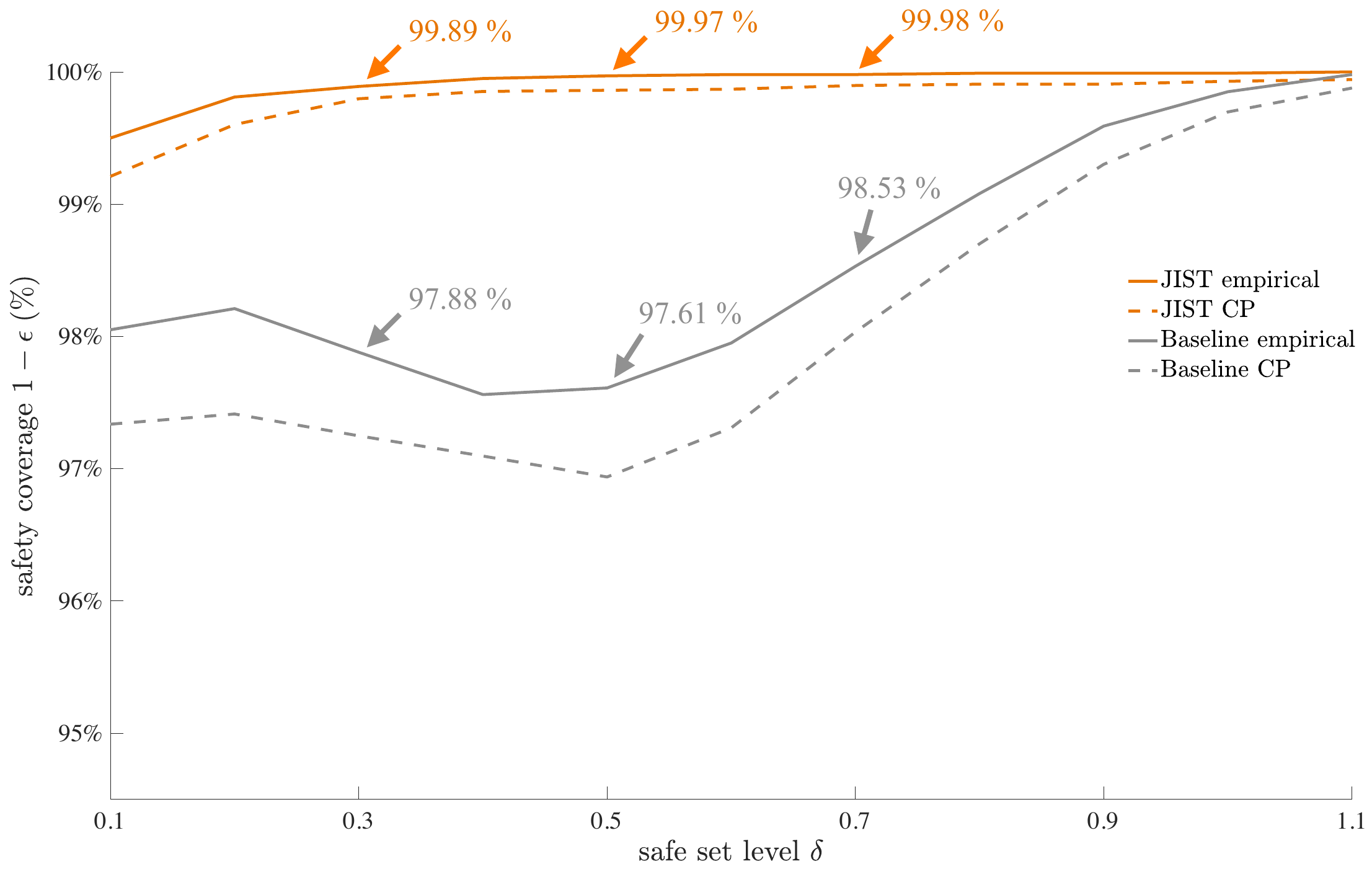}
  \caption{\label{fig:sr}Empirical and certified safe rates of \beliefsf~computed from $20000$ randomized trials. For a fixed safety level $\delta$, \jist~yields a tighter safety coverage than the \directcp~baseline.
  }
\end{figure}

Next, we compare rejection rates of \jist~and \directcp~under the same safety coverage when the verified \beliefsf~is deployed for online safe control in information space (\autoref{fig:front_fig}).
Given a fixed safe set level $\delta$, recall that \beliefsf~with \jist~rejects a trial when the information state is outside of the learned safe set ($\valfunc(\jstate_0) < \delta$) or the predicted inference accuracy is too low ($h_L(\jstate_0) < 0$).
In case of \directcp, a trial is rejected if $\valfunc(\jstate_0) < \delta_\baseline$, where $\delta_\baseline \geq \delta$ is a larger, and hence more restrictive safe set level that achieves the same safety coverage ($1-\epsilon$) as that of \jist.
We plot rejection rates with varying safe set levels $\delta$ in \autoref{fig:rr}.
With \jist, the rejection rate is as low as the false inference rate, since the learned \beliefsf~is generally capable of enforcing safety when the inference algorithm correctly predicts the opponent's control bound.
On the other hand, \beliefsf~verified with \directcp~requires a substantially larger safe set level $\delta_\baseline$ to achieve the same safety coverage, ultimately leading to a much higher rejection rate.
Finally, we visualize the planar ($x-y$ positions) slices of level sets with $\delta_\jist$ and $\delta_\baseline$ that achieve the same safety coverage $\epsilon$ in \autoref{fig:safeset}.

\begin{figure}[!hbtp]
  \centering
  \includegraphics[width=1.0\columnwidth]{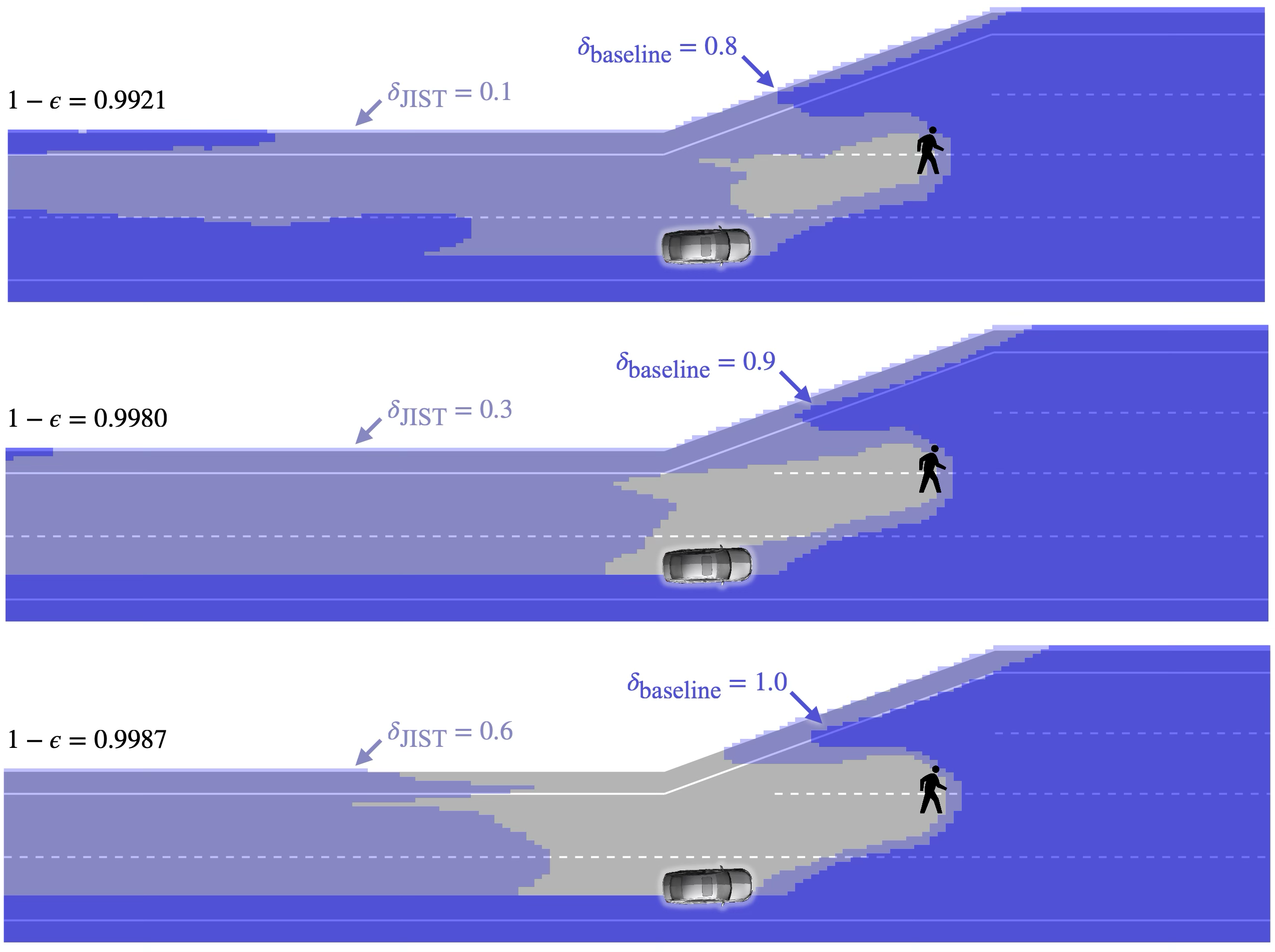}
  \caption{\label{fig:safeset}Planar slices of level sets with $\delta_\jist$ and $\delta_\baseline$ under safety coverage $\epsilon$.
  }
\end{figure}

\section{Limitations and Future Work}
\label{sec:limitation}

Our approach improves the filter's permissiveness by focusing verification and deployment on a region where inference is expected to be reliable.
This benefit hinges on having access to a reasonable inference score function $h_L(\cdot)$ and on the inference algorithm being sufficiently accurate.
When these conditions do not hold, a practical mitigation is a layered safety architecture in which the robot falls back to a conservative, non-belief safety filter whenever inference quality is low.

In line with standard conformal prediction, our safety guarantees are distributional rather than worst-case.
In particular, our analysis relies on standard exchangeability and i.i.d. assumptions as well as the independence of distribution from agent actions (Assumption \ref{assump:ds}).
When agent interactions induce a nontrivial distribution shift, these guarantees may no longer hold.
Recent advances in adaptive \gls{CP} that explicitly account for interaction-induced distribution shift offer a promising direction for strengthening our framework~\cite{mirzaeedodangeh2025safe,binny2025moved}.

Another promising direction worth exploring is to close the loop between verification and filter synthesis by designing an iterative regime that uses \gls{CP} outcomes to improve both the safety filter and the inference module.
Concretely, the \gls{CP}-certified safe set can guide the tuning of \beliefsf~parameters, \eg, reducing $\epsilon_\theta$ in \eqref{eq:pred_ctrl_bd} in regions where inference is accurate and uncertainty is low.
This suggests a verification-aware filter synthesis pipeline in which \gls{CP} is not merely a post-hoc certificate, but also provides supervision signals that actively inform filter design and data collection.
We are also eager to validate our proposed framework on \beliefsf~stacks with more complex inference modules, \eg, LSTM or Transformer-based predictors~\cite[Sec.~4.3]{hu2023deception}.

\section{Conclusions}
\label{sec:conclusions}

In this paper, we studied the problem of verifying closed-loop safety for belief-space safety filters in interactive planning settings where both runtime inference and learned safety filters can be imperfect.
Our key idea is to couple verification with inference quality: rather than certifying the filtered system end-to-end, we verify and deploy the filter within a region where the robot's inference is expected to be reliable.
This approach preserves the simplicity of conformal prediction while avoiding unnecessary conservativeness that arises when rare inference failures are conflated with errors of the safety filter itself.
In a simulated vehicle--pedestrian interaction benchmark, our proposed verification algorithm certified a substantially more permissive filter than a standard conformal prediction baseline, achieving higher safety coverage at the same base safe-set level and lower rejection rate when deployed online.


%
%
%
%

\bibliography{references,_acronyms}
\bibliographystyle{splncs04}

\end{document}

%% file: headers/commands.tex
\usepackage{multicol}
\usepackage[bookmarks=true]{hyperref}

\usepackage{float}
\usepackage{amsfonts}
\usepackage{comment}
\usepackage{times}
\usepackage{amsmath}
\usepackage{changepage}
\usepackage{amssymb}
\usepackage{graphicx}
\usepackage{adjustbox}
\usepackage{latexsym}
\usepackage{enumerate}
\usepackage[font=small,belowskip=0pt]{caption}
\usepackage{mathtools}
\usepackage{algcompatible}
\usepackage{algpseudocode}
\usepackage[ruled,vlined,linesnumbered]{algorithm2e}
\usepackage{bbm}
\usepackage{xcolor}
\usepackage{bm}
\usepackage{booktabs}
\usepackage{comment}
\usepackage[xindy, acronyms, nowarn]{glossaries}   %

\usepackage{fontawesome5}

\definecolor{jhublue}{HTML}{002D72} 
\definecolor{lavender}{HTML}{9E8FB0}
\definecolor{porange}{HTML}{E77500} 

\SetCommentSty{mycommfont}

\usepackage[breakable]{tcolorbox}
\usepackage{setspace}
\definecolor{exampleback}{rgb}{0.6, 0.6, 0.6}
\newenvironment{runningexample}[0]
{   \singlespacing
    \begin{tcolorbox}[standard jigsaw,center,breakable,colframe=black,width=\linewidth,opacityback=0,colbacktitle=exampleback,boxrule=0.4pt]
    \small
}
{\end{tcolorbox}}

\usepackage{url}
\usepackage{hyperref}
 \hypersetup{
     colorlinks=true,
     linkcolor=blue,
     filecolor=blue,
     citecolor=blue,      
     urlcolor=black,
     }
\usepackage{diagbox}

%% file: headers/notation.tex
\usepackage{fontawesome5}  

\newcommand{\p}[1]{\smallskip \noindent \textbf{{#1}.}}

\newcommand{\reals}{\mathbb{R}}

\newcommand{\compl}{\text{c}}  

\renewcommand{\emptyset}{{\text{\O}}}  

\newcommand{\distr}{{\mathcal{D}}}
\newcommand{\prob}{{\mathbb{P}}}

\newcommand{\belspace}{\Delta}

\newcommand{\state}{{x}}
\newcommand{\jstate}{{z}}
\newcommand{\jsset}{\mathcal{Z}}

\newcommand{\ctrl}{{u}}

\newcommand{\obsspace}{{\mathcal{Y}}}

\newcommand{\belstate}{b}
\newcommand{\beltraj}{\textbf{\belstate}}

\newcommand{\statespace}{\mathcal{X}}

\newcommand{\traj}{{\mathbf{\state}}}
\newcommand{\jtraj}{{\mathbf{\jstate}}}

\newcommand{\xset}{{\mathcal{X}}}

\newcommand{\cset}{{\mathcal{U}}}

\newcommand{\csetest}{\widehat{\cset}}

\newcommand{\nx}{{n_\state}}

\newcommand{\me}{{m_\ego}}
\newcommand{\mo}{{m_\oppo}}

\newcommand{\dyn}{{f}} 

\newcommand{\beldyn}{f_L}

\newcommand{\jntdyn}{F}

\newcommand{\valfunc}{{V}}

\newcommand{\policy}{{\pi}}

\newcommand{\tpolicy}{\tilde{\policy}}

\newcommand{\failure}{{\mathcal{F}}}
\newcommand{\failureset}{\failure}

\newcommand{\safeSet}{{\Omega}}
\newcommand{\safeset}{{\safeSet}}
\newcommand{\safeSetMax}{\safeSet^*}

\newcommand{\ego}{{e}}

\newcommand{\oppo}{{o}}

\newcommand{\cost}{{J}}
\newcommand{\costinf}{{\cost_L}}

\newcommand{\shield}{\text{\tiny{\faShield*}}}

\newcommand{\task}{{\text{task}}}

\newcommand{\policyTask}{{\policy^\task}}

\newcommand{\safetyFilter}{{\phi}}

\newcommand{\ms}{~\text{m}/\text{s}}

\newcommand{\obs}{{y}}

\newcommand{\beliefsf}{{\textsc{BeliefSF}}}
\newcommand{\directcp}{{\textsc{DirectCP}}}
\newcommand{\jist}{{\textsc{JIST}}}

\newcommand{\type}{\text{class}}
\newcommand{\goal}{\text{goal}}
\newcommand{\ped}{\text{ped}}
\newcommand{\seg}{\text{seg}}
\newcommand{\eg}{\text{e.g.}}
\newcommand{\ie}{\text{i.e.}}
\newcommand{\baseline}{\text{baseline}}

\newcommand{\princeton}[1]{\ifthenelse{\boolean{include-notes}}{\textcolor{orange}{#1}}{}}

%% file: headers/gls.tex
\glsdisablehyper
\makeglossaries

\newglossaryentry{JIST}
{
  name={JIST},
  description={Joint Inference-Safety Test},
  first={Joint Inference-Safety Test (\glsentrytext{JIST})},
}

\newglossaryentry{CP}
{
  name={CP},
  description={conformal prediction},
  first={conformal prediction (\glsentrytext{CP})},
}

\newglossaryentry{LSE}
{
  name={LSE},
  plural={LSEs},
  description={local Stackelberg equilibrium},
  first={local Stackelberg equilibrium (\glsentrytext{LSE})},
  descriptionplural={local Stackelberg equilibria},
  firstplural={local Stackelberg equilibria (\glsentryplural{LSE})}
}

\newglossaryentry{FSE}
{
  name={FSE},
  plural={FSEs},
  description={feedback Stackelberg equilibrium},
  first={feedback Stackelberg equilibrium (\glsentrytext{FSE})},
  descriptionplural={feedback Stackelberg equilibria},
  firstplural={feedback Stackelberg equilibria (\glsentryplural{FSE})}
}

\newglossaryentry{DSE}
{
  name={DSE},
  plural={DSEs},
  description={differential Stackelberg equilibrium},
  first={differential Stackelberg equilibrium (\glsentrytext{DSE})},
  descriptionplural={differential Stackelberg equilibria},
  firstplural={differential Stackelberg equilibria (\glsentryplural{DSE})}
}

\newglossaryentry{RL}
{
  name={RL},
  description={reinforcement learning},
  first={reinforcement learning (\glsentrytext{RL})},
}

\newglossaryentry{HJI}
{
  name={HJI},
  description={Hamilton--Jacobi--Isaacs},
  first={Hamilton--Jacobi--Isaacs (\glsentrytext{HJI})},
}

\newglossaryentry{GAS}
{
  name={GAS},
  description={globally asymptotically stable},
  first={globally asymptotically stable (\glsentrytext{GAS})},
}

\newglossaryentry{a.s.}
{
  name={a.s.},
  description={almost surely},
  first={almost surely (\glsentrytext{a.s.})},
}

\newglossaryentry{MDP}
{
  name={MDP},
  description={Markov decision process},
  first={Markov decision process (\glsentrytext{MDP})},
}

\newglossaryentry{SAC}
{
  name={SAC},
  description={Soft Actor--Critic},
  first={Soft Actor--Critic (\glsentrytext{SAC})},
}

\newglossaryentry{tGDA}
{
  name={$\tau$-GDA},
  description={gradient descent-ascent with finite timescale separation},
  first={gradient descent-ascent with finite timescale separation (\glsentrytext{tGDA})},
}

\newglossaryentry{MAGICS}
{
  name={MAGICS},
  description={Minimax Actors Guided by Implicit Critic Stackelberg},
  first={Minimax Actors Guided by Implicit Critic Stackelberg (\glsentrytext{MAGICS})},
}

\newglossaryentry{ISAACS}
{
  name={ISAACS},
  description={Iterative Soft Adversarial Actor--Critic for Safety},
  first={Iterative Soft Adversarial Actor--Critic for Safety (\glsentrytext{ISAACS})},
}

\newglossaryentry{ISAACStbg}
{
  name={ISAAC$\mathcal{S}$},
  description={Iterative Soft Adversarial Actor--Critic with Stackelberg learning dynamics for Safety},
  first={Iterative Soft Adversarial Actor--Critic with Stackelberg learning dynamics for Safety (\glsentrytext{ISAACStbg})},
}

\newglossaryentry{A2C}
{
  name={A2C},
  description={Advantage Actor--Critic},
  first={Advantage Actor-Critic (\glsentrytext{A2C})},
}

\newglossaryentry{iHvp}
{
  name={iHvp},
  description={inverse hessian vector product},
  first={inverse-Hessian-vector product
  (\glsentrytext{iHvp})},
}

\newglossaryentry{jvp}
{
  name={jvp},
  description={jacobian vector product},
  first={Jacobian-vector product
  (\glsentrytext{jvp})},
}

\newglossaryentry{GSE}
{
  name={GSE},
  description={global Stackelberg equilibrium},
  first={global Stackelberg equilibrium (\glsentrytext{GSE})}
}

\newglossaryentry{BNP}
{
  name={B\&P},
  description={Branch-and-Play},
  first={Branch-and-Play (\glsentrytext{BNP})},
}

\newglossaryentry{LQ}
{
  name={LQ},
  description={linear quadratic},
  first={linear quadratic (\glsentrytext{LQ})},
}

\newglossaryentry{ILQR}
{
  name={ILQR},
  description={iterative linear quadratic regulator},
  first={iterative linear quadratic regulator (\glsentrytext{ILQR})},
}

\newglossaryentry{ATC}
{
  name={ATC},
  description={air traffic control},
  first={air traffic control (\glsentrytext{ATC})},
}

\newglossaryentry{STP}
{
  name={STP},
  description={sequential trajectory planning},
  first={sequential trajectory planning (\glsentrytext{STP})},
}

\newglossaryentry{FCFS}
{
  name={FCFS},
  description={first-come-first-served},
  first={first-come-first-served (\glsentrytext{FCFS})},
}

\newglossaryentry{MAPF}
{
  name={MAPF},
  description={multi-agent pathfinding},
  first={Multi-agent pathfinding (\glsentrytext{MAPF})},
}

\newglossaryentry{MPC}
{
  name={MPC},
  description={model predictive control},
  first={model predictive control (\glsentrytext{MPC})},
}

\newglossaryentry{NOD}
{
  name={NOD},
  plural={NOD},
  description={nonlinear opinion dynamics},
  first={nonlinear opinion dynamics (\glsentrytext{NOD})},
  descriptionplural={nonlinear opinion dynamics},
  firstplural={nonlinear opinion dynamics (\glsentryplural{NOD})}
}

\newglossaryentry{DNN}
{
  name={DNN},
  description={deep neural network},
  first={deep neural network (\glsentrytext{DNN})},
}

\newglossaryentry{MLE}
{
  name={MLE},
  description={maximum likelihood estimation},
  first={maximum likelihood estimation (\glsentrytext{MLE})},
}

\newglossaryentry{ODG}
{
  name={ODG},
  description={opinion-guided dynamic game},
  first={opinion-guided dynamic game (\glsentrytext{ODG})},
}

\newglossaryentry{MLP}
{
  name={MLP},
  description={multilayer perceptron},
  first={multilayer perceptron (\glsentrytext{MLP})},
}

\newglossaryentry{E2E-BC}
{
  name={E2E-BC},
  description={end-to-end behavior cloning},
  first={End-to-end behavior cloning (\glsentrytext{E2E-BC})},
}

\newglossaryentry{HJ}
{
  name={HJ},
  description={Hamilton--Jacobi},
  first={Hamilton--Jacobi (\glsentrytext{HJ})}
}

\newglossaryentry{DCBF}
{
  name={DCBF},
  description={discrete-time control barrier function},
  first={discrete-time control barrier function (\glsentrytext{DCBF})}
}

\newglossaryentry{CBF}
{
  name={CBF},
  description={control barrier function},
  first={control barrier function (\glsentrytext{CBF})}
}

\newglossaryentry{Q-CBF}
{
  name={Q-CBF},
  description={state--action control barrier function},
  first={state--action control barrier function (\glsentrytext{Q-CBF})}
}

\newglossaryentry{ODD}
{
  name={ODD},
  description={Operational Design Domain},
  first={operational design domain (\glsentrytext{ODD})}
}

\newglossaryentry{LRSF}
{
  name={LRSF},
  description={Last-Resort Safety Filter},
  first={last-resort safety filter (\glsentrytext{LRSF})}
}

\newglossaryentry{HCSF}
{
  name={HCSF},
  description={Human-Centered Safety Filter},
  first={human-centered safety filter (\glsentrytext{HCSF})}
}

\newglossaryentry{NLP}
{
  name={NLP},
  description={nonlinear programming problem},
  first={nonlinear programming problem (\glsentrytext{NLP})},
}

\newglossaryentry{AC}
{
  name={AC},
  description={Assetto Corsa},
  first={Assetto Corsa (\glsentrytext{AC})},
}

\newglossaryentry{ANOVA}
{
    name={ANOVA},
    description={Analysis of Variance},
    first={analysis of variance (\glsentrytext{ANOVA})}
}

\newglossaryentry{SME}
{
    name={SME},
    description={Simple Main Effects},
    first={simple main effects (\glsentrytext{SME})},
}

\newglossaryentry{HSD}
{
    name={HSD},
    description={Honestly Significant Difference},
    first={honestly significant difference (\glsentrytext{HSD})},
}

\newglossaryentry{ECDF}
{
    name={ECDF},
    description={Empirical Cumulative Distribution Function},
    first={empirical cumulative distribution function (\glsentrytext{ECDF})},
}

\newglossaryentry{OCP}
{
    name={OCP},
    description={Optimal Control Problem},
    first={optimal control problem (\glsentrytext{OCP})}
}

\newglossaryentry{AI}
{
    name={AI},
    description={Artificial Intelligence},
    first={artificial intelligence (\glsentrytext{AI})}
}

\newglossaryentry{HRI}
{
    name={HRI},
    description={Human--Robot Interaction},
    first={human--robot interaction (\glsentrytext{HRI})}
}